\documentclass[11pt]{article}

\usepackage{chihao}
\usepackage{geometry}

\usepackage{enumerate}

\newtheorem{fact}{Fact}

\newcommand{\DKL}{\!{KL}\xspace}
\newcommand{\ind}[1]{\bb{1}_{#1}}
\newcommand{\Ber}{{\!{Ber}}\xspace}
\newcommand{\findbest}{\textsc{FindBest}\xspace}
\newcommand{\me}{\textsc{MedianElimination}\xspace}
\renewcommand{\a}{\!{arm}}
\newcommand{\fd}{\mathfrak{d}}

\title{On the Problem of Best Arm Retention}
		
		\author{Houshuang Chen\\Shanghai Jiao Tong University\\\textsf{chenhoushuang@sjtu.edu.cn} \and Yuchen He \\ Shanghai Jiao Tong University\\ \textsf{yuchen\_he@sjtu.edu.cn}\and Chihao Zhang\\ Shanghai Jiao Tong University\\ \textsf{chihao@sjtu.edu.cn}}
\begin{document}	
\maketitle	
		\begin{abstract}
			This paper presents a comprehensive study on the problem of Best Arm Retention (BAR), which has recently found applications in streaming algorithms for multi-armed bandits. In the BAR problem, the goal  is to retain $m$ arms with the best arm included from $n$ after some trials, in stochastic multi-armed bandit settings. We first investigate pure exploration for the BAR problem under different criteria, and then minimize the regret with specific constraints, in the context of further exploration in streaming algorithms.
			\begin{itemize}
				\item We begin by revisiting the lower bound for the $(\eps,\delta)$-PAC algorithm for Best Arm Identification (BAI) and adapt the classical KL-divergence argument to derive optimal bounds for $(\eps,\delta)$-PAC algorithms for BAR. 
				\item We further study another variant of the problem, called $r$-BAR, which requires the expected gap between the best arm and the optimal arm retained is less than $r$. We prove tight sample complexity for the problem. 
				\item We explore the regret minimization problem for $r$-BAR and develop algorithm beyond pure exploration. We conclude with a conjecture on the optimal regret in this setting.
			\end{itemize}
		\end{abstract}
		\newpage
		
		
		

\setcounter{tocdepth}{1}		
\tableofcontents

\section{Introduction}
	The multi-armed bandit (MAB) framework, pioneered by \cite{RH52}, has emerged as a powerful paradigm for modeling sequential decision-making under uncertainty in various real-world applications, ranging from clinical trials to online advertising.  Among myriad MAB problems, the \emph{Best Arm Identification} (BAI), as the pure exploration version of stochastic MAB, stands out as a critical task, where the objective is to identify the best arm based on their rewards.
		
		At the beginning of the stochastic MAB game, the player confronts $n$ arms, each associated with an unknown distribution. At each round $t\in[T]$\footnote{ $T$ can be a stopping time.}, the player  chooses an arm and receives a reward. The player is trying to obtain higher accumulated rewards. Equivalently, the goal is to minimize the expected regret, which is the expected accumulated reward  difference between  playing the best arm with the highest mean and playing with the algorithm chosen arms. In contrast, BAI, the pure exploration version of MAB, seeks to swiftly identify the arm with the highest mean, ignoring reward considerations during the decision-making process.
		
		Recent research has explored streaming algorithms~\cite{AW22,HYZ24} employing multiple passes to retain $m$ arms due to memory constraints.  This particular setting naturally leads to the \emph{Best Arm Retention} (BAR) problem, a pragmatic extension of BAI accommodating scenarios with limited memory or computational resources. The name of the problem was coined in~\cite{HYZ24} and was also known as ``arm trapping problem'' in literature~\cite{AW22}. However, previous study for the problem is either incomplete or suboptimal.
		
		In BAR, the objective shifts from identifying the arm with the highest expected reward to retain a subset of size 
		$m$ containing the best arm for further exploration or exploitation. In practice, this subset may be subject to constraints like fixed memory capacity, making BAR an adaptable framework for addressing real-world considerations such as uncertainty, dynamic environments, and regret minimization over time. Notably, BAR reduces to the classic BAI problem when $m=1$, and becomes easier as $m$ increases.
		
		In this work, we call an arm $\eps$-optimal  if  the mean gap between the best arm and this arm is less than $\eps$. Similar to the PAC algorithm for BAI,  it is natural to investigate the $(\eps,\delta)$-BAR problem, the PAC setting of BAR, where the objective is to ensure that the set of $m$ retained arms contains  an $\eps$-optimal arm  with at least $1-\delta$ probability after observing as few samples as possible. The least number of samples to fulfill the requirement is called the sample complexity of $(\eps,\delta)$-BAR. The tight bound for this complexity is provided in   \Cref{thm:bar}.
		
		\begin{theorem}\label{thm:bar}
			For any $(\eps,\delta)$-PAC algorithm for BAR satisfying $\eps\leq \frac{1}{8}$ and $\delta\leq \frac{n-m}{n}(1-\beta)$, where $\beta\in (0,1)$ is a universal constant, the sample complexity is 
			$$\Theta\tp{\frac{n-m}{\eps^2}\log\frac{n-m}{n\delta}}.$$
		\end{theorem}
		It is trivial that the sample complexity is zero when $\delta\geq\frac{n-m}{n}$ because we can choose $m$ arms uniformly at random. In fact,  the lower bound in \Cref{thm:bar} addresses almost all feasible $\delta$ except $\frac{n-m}{n}-\delta=o\tp{\frac{1}{n}}$, as explained in \Cref{remark}. 
		
		When $m=1$, the above theorem can be applied to the BAI problem,  yielding a sample complexity  bound for $\delta\geq 0.5$. To the best of our knowledge, previous methods in literature  could only handle the case where $\delta< 0.5$.
		\begin{corollary}
			For any $(\eps,\delta)$-PAC algorithm for BAI satisfying $\eps\leq \frac{1}{8}$ and $\delta\leq \frac{n-1}{n}(1-\beta)$, where $\beta\in (0,1)$ is a universal constant, the sample complexity is 
			$$\Theta\tp{\frac{n-1}{\eps^2}\log\frac{n-1}{n\delta}}.$$
		\end{corollary}

		Recall that the  BAR algorithm is used to retain $m$ arms for further task. If the objective of the subsequent task is to minimize regret,  achieving $(\eps,\delta)$-PAC optimality  is not strictly necessary.  As noted in \cite{HYZ24}, it suffices for the expected gap between the mean of the best and the optimal in the retained arm to be small. That is a weaker requirement than $(\eps,\delta)$-PAC learnability. To capture the complexity of this requirement, we define the problem $r$-BAR, where the goal is to guarantee the expected gap is less than $r$. As before, when $m=1$, it is equivalent to identifying an arm whose mean is at most $r$ from the optimal one in expectation. We call it $r$-BAI problem, which has been investigated in~\cite{BMS09}. We determine the sample complexity of this problem in \Cref{thm:rbar-sample}.

		\begin{theorem}\label{thm:rbar-sample}
			The sample complexity of $r$-BAR is 
			$$\Theta\tp{\frac{(n-m)^3}{(nr)^2}}.$$
		\end{theorem}

		The motivation for introducing $r$-BAR is to minimize regret for the subsequent task. It is also natural to consider the incurred regret of the BAR algorithm itself, beyond pure exploration. To this end, we introduce a new complexity measure called \emph{regret complexity}, which intuitively measures how much regret one has to pay to retain an arm whose expected mean reward is at most $r$ from the best. The formal definition of the regret complexity is in \Cref{sec:prelim-BAR}, and we present the corresponding bounds in \Cref{thm:rbar-regret}.

		\begin{theorem}\label{thm:rbar-regret}
			There exists an algorithm for $r$-BAR such that the regret complexity is no more than
			\[
			\+O\tp{\frac{(n-m)^2}{nr}\tp{1+\sqrt{\frac{m}{n-m}}}},
			\]
			and for any algorithm, the regret complexity is no less than
			$$\Omega\tp{\frac{(n-m)^2}{nr}}.$$
		\end{theorem}
		The gap between the upper and the lower bounds is $\tp{1+\sqrt{\frac{m}{n-m}}}$. Thus our bounds are tight for $n-m=\Omega(n)$. When $m$ is very close to $n$, the gap is $\sqrt{n}$, and we will explain in \Cref{ssec:diff} that  eliminating this gap is not easy because different instances requires different sample size and therefore a more sophisticated adaptive strategy is required for an optimal algorithm.

		\subsection*{Related Work}

		Pure exploration in stochastic MAB (see \cite{LS20})  has garnered significant attention and has been explored in various settings, with the most prominent being Best Arm Identification (BAI). Research in this area has primarily focused on two key frameworks: fixed confidence and fixed budget. 
		
		The sample complexity under fixed confidence, where the goal is to identify the best arm with a given confidence level, has been extensively studied by works \cite{EMM06,JMNB14,GK16,CLQ17,WTP21,Russo16}.  Several algorithms have been developed, such as \emph{successive elimination} \cite{EMM06}, \emph{upper confidence bound} \cite{GGL12,KTAS12,JMNB14,KK13}, \emph{exponential-gap elimination} \cite{KKS13} and the best known algorithm to date \cite{CLQ17}.

		In contrast, the fixed-budget setting, where the objective is to maximize the probability of identifying the best arm within a limited number of samples, has also received significant attention. Pioneering works by \cite{MLG04,BMS09,ABM10,KKS13,CL16, KCG16} have provided insights into the trade-offs between budget constraints and success probability.
		
		The optimal bounds for both above settings are still unknown. An easier setting is to output an $\eps$-optimal arm instead of the best arm.  The fixed budget under this setting is to maximize the probability of outputting the $\eps$-optimal arm \cite{KJ20,ZSSJ23}.   For the fixed confidence setting, it known as 
		$(\eps,\delta)$-PAC (Probably Approximately Correct) algorithms for BAI, aiming to identify an 
		$\eps$-optimal arm with fixed confidence $\delta$, has been well studied by~\cite{MT04,HR22,JDK23}. 
		
		These settings naturally extend to finding the top 
		$m$ arms, a problem that has been explored by \cite{KS10,KTAS12,YQWY23,CHFL08,SJR17}. However, this extension poses greater complexity and in practice, retaining the best arm alone often suffices. For instance, if one would like to perform some regret minimization algorithm on the $m$ arms, retaining the optimal one already yields optimal regret. Notably, $(\eps,\delta)$-PAC algorithm requires $\Omega\tp{\frac{n}{\eps^2}\log\frac{m}{\delta}}$ samples to retain the top $m$ arms  \cite{KTAS12},  which is worse than our bounds for BAR in \Cref{thm:bar}.  Moreover, structured arm groups, where arms exhibit dependencies or are organized in a combinatorial structure, have also been investigated. Notable works in this area include \cite{CLKLC14,KKR23,SWX24,DMSV20,FJJR19,LGC16,MJTN20}.
		
		The concept of \emph{Best Arm Retention} (BAR) was introduced by  \cite{HYZ24},  which extends the traditional BAI framework by focusing on retaining an 
		$\eps$-best arm with fixed confidence. The idea of ``trapping the best arm''   in a similar context first appeared in the work of \cite{AW22} in the domain of stream algorithms. Additionally, the notion of $r$-BAI, a special case of $r$-BAR when $m=1$ and often referred to as ``simple regret'' has been discussed in earlier works by \cite{BMS09}. Very recently, the work \cite{YZZ25} applied technique developed in this paper to establish a tight gap-dependent lower bound for single-pass streaming stochastic MAB.

		To the best of our knowledge, this paper is the first systematic investigation of the BAR problem.
		
		\section{Notation and Preliminaries}
		For any integer $n>0$, let $[n]$ denote the set $\set{1,2,\dots,n}$. If $x,y\in\bb R$ and $x\leq y$, $[x,y]$ denotes the closed interval $\set{z:x\leq z\leq y}$.
		Let $\Ber(x)$ denote the Bernoulli distribution with mean $x\in[0,1]$. 
		The  Kullback-Leibler (KL) divergence between two Bernoulli distributions with means $x$ and $y$ is given by 
		$\fd(x,y)\defeq x\log\frac{x}{y}+(1-x)\log\frac{1-x}{1-y}$  for brevity.
		There are some properties of the KL divergence:
		\begin{fact}\label{fact:kl}
			\begin{enumerate}[(a)]
				\item $\mathfrak{d}(\cdot, y)$ (or $\fd(x,\cdot)$) is convex for any fixed $y$ (or $x$);
				\item for any $0\leq a\leq x\leq y\leq b\leq 1$,  $\fd(a,b)\geq \fd(x,y)$;
			\end{enumerate}
		\end{fact}
		\begin{proof}
			\begin{enumerate}[(a)]
				\item Directly calculating    $\frac{\partial^2 \fd(x,y) }{\partial x^2}\geq 0$ implies $\fd(\cdot, y)$ is convex and $\fd(x,\cdot)$ is similar. 
				\item Since $\left. \frac{\partial{\fd(x,y)}}{\partial x}\right\vert_{x=y}=0$, $\fd(\cdot,y)$ achieves the minimum in $x=y$. Similarly, $\fd(x,y)$ is the minimum of $\fd(x,\cdot)$. Therefore, $\fd(a,b)\geq \fd(x,b)\geq \fd(x,y)$.
			\end{enumerate}
		\end{proof}
		
		The following fact will assist in  bounding the KL divergence.
		\begin{fact}[ \cite{Top07} ]\label{fact:log-ineq}
			The following inequalities hold.
			\begin{enumerate}[(a)]
				\item $\log(1+x)\geq \frac{x}{1+x},  \forall x>-1$;
				\item $\log(1+x)\geq \frac{x}{1+x}(1+\frac{x}{2+x})=\frac{2x}{2+x},  \forall x>0$;
				\item 	$\log(1+x)\geq\frac{x}{1+x}\frac{2+x}{2},  \text{ if } -1<x\leq 0$.
			\end{enumerate}
		\end{fact}
		
		The following inequality will be utilized in the proof of \Cref{lem:likely-ratio}.
		\begin{fact}[Log Sum Inequality]\label{fact:log-sum}
			Let $a_1,\dots ,a_n$ and  $b_1,\dots ,b_n$ be nonnegative numbers, and define $A=\sum_{i=1}^na_i$ and  $B=\sum_{i=1}^nb_i$. Then 
			$$\sum_{i=1}^na_i\log\frac{a_i}{b_i}\geq A\log \frac{A}{B},$$
			where equality holds if and only if $\frac{a_i}{b_i}$ are equal for all $i$.
		\end{fact}
		
		\subsection{Multi-Armed Bandits}\label{ssec:mab}
		
		In this paper, we exclusively consider the stochastic \emph{Multi-Armed Bandit} (MAB) problem, which can be represented by an $n$-dimensional product distribution $\nu = (\nu_1, \nu_2, \dots, \nu_n)$. Each distribution corresponds to an arm. At each round/day $t \in [T]$, the player selects an arm $a_t \in [n]$ and receives a reward $r_t \sim \nu_{a_t}$ independently.
		
		Let $\mu = (\mu_1, \mu_2, \dots, \mu_n)$ denote the mean vector of $\nu$. Define $i^* = \arg\max_{i \in [n]} \mu_i$ as the best arm with the highest mean, and let $\Delta_i = \mu_{i^*} - \mu_i$ represent the mean gap between the best arm and arm $i$. Additionally, let $T_i = \sum_{t=1}^T \ind{a_t = i}$ denote the number of times arm $i$ is pulled. The player's objective is to maximize the accumulated reward $\sum_{t=1}^T r_t$, or equivalently, to minimize the regret, defined as $$R(n,T)=T\mu_{i^*}-\E{\sum_{t=1}^T r_t}=\E{\sum_{i=1}^n\Delta_i T_i},$$ which measures the difference between the accumulated expected reward of the best arm and that of the algorithm. We abbreviate $R(n,T)$ as $R$ when the context is clear and refer to a product distribution $\nu$ as an MAB instance. If each $\nu_i$ is a Bernoulli distribution, we also use $\mu$ to denote an MAB instance.
		
		For the regret of the MAB problem, there exist several algorithms that achieve tight bounds of $\Theta\left(\sqrt{nT}\right)$ up to a constant factor, such as the \emph{Online Stochastic Mirror Descent} (OSMD) or the \emph{Follow the Regularized Leader} (FTRL) algorithm. The following algorithm from \cite{LG21} provides a refined constant factor.
		
		Let $\Delta_{(n-1)}\defeq\allowbreak\set{\*q\in \bb R_{\geq 0}\cmid \sum_{i=1}^n \*q(i)=1}$ denote the probability simplex with dimension $n-1$. Here, $\*q(i)$ represents the value at the $i$-th position of vector $\*q$. Consider a function $F:\bb R^n\to \bb R\cup\set{\infty}$. The Bregman divergence with respect to $F$ is defined as $B_F(\*q,\*p) = F(\*q)-F(\*p) - \inner{\nabla F(\*p)}{\*q-\*p}$ for any $\*q,\*p\in \bb R^n$.
		
		The algorithm proposed in~\cite{LG21} is designed for loss cases, where each pull results in a loss associated with the corresponding arm instead of a reward. To adapt their algorithm to our setting, we can perform a simple reduction by constructing the loss of each arm $\ell_t(i)$ as $1-r_t(i)$, where $r_t(i)$ is the reward of arm $\a_i$. It is straightforward to verify that the results in~\cite{LG21} also hold for the reward setting. Let $\eta$ be the learning rate and $F:\bb R^{\abs{S}}\to \bb R\cup\set{\infty}$ be the potential function, where $S$ is the arm set. Without loss of generality, we index the arms in $S$ by $[\abs{S}]$.
		
		\begin{algorithm}[h]
			\caption{Online Stochastic Mirror Descent}
			\label{algo:OSMD}
			\Input{a set of arms $S$ and the number of rounds $L$}
			\begin{algorithmic}[1]
				\Procedure{\textsc{MirrorDescent}}{{$S,L$}}
				\State $Q_1\gets \arg\min_{\*q\in \Delta_{(\abs{\+S}-1)}} F(\*q)$\;
				\For{$t=1,2\dots, L$}
				\State Sample arm $a_t\sim Q_t$, observe reward $r_t(a_t)$ and let $\ell_t(a_t) = 1-r_t(A_t)$\;
				\State Compute reward estimator $\hat \ell_t$ as 
				\begin{align*}
					\hat \ell_t(i) &= \1{A_t=i}\tp{\ell_t(i) -\frac{1}{2} + \frac{\eta}{8}\tp{1+\frac{1}{Q_t(i) + \sqrt{Q_t(i)}}}} \\
					&\quad- \frac{\eta Q_t(A_t)}{8\tp{ Q_t(i) + \sqrt{Q_t(i)} }}
				\end{align*}
				\State Set $Q_{t+1} = \arg\min_{\*q\in \Delta_{(\abs{\+S}-1)}} \inner{\*q}{\hat\ell_t} + \frac{1}{\eta}\cdot B_F(\*q,Q_t)$\;
				\EndFor
				\EndProcedure
			\end{algorithmic}
		\end{algorithm}
		By choosing $\eta = \sqrt{\frac{8}{L}}$ and $F(\*q) = -2\sum_{i=1}^{\abs{S}} \sqrt{\*q(i)}$, the following proposition can be directly derived from Theorem 11 in~\cite{LG21}.

		\begin{proposition}[\cite{LG21}, Theorem 11]\label{prop:osmd}
			Running the  \Cref{algo:OSMD} on any MAB instance with above parameters, the regret is at most $R(n,T) \leq \sqrt{2nT}$.
		\end{proposition}
		
		\subsection{Best Arm Identification}
		
		Given an MAB instance $\nu$, the \emph{Best Arm Identification} (BAI) problem aims to identify the arm $i^*$ with the highest mean based on as few samples as possible. Different from the fixed $T$ in MAB, the sample size $T$ here is a stopping time with respect to the filtration $\tp{\+F_t}_{t\in\bb N}$ where $\+F_t=\sigma\tp{a_1,r_1,a_2,r_2,\dots,a_t,r_t}$. In our paper, we focus on  the $(\eps, \delta)$-PAC setting of BAI, denoted as $(\eps, \delta)$-BAI, which requires the algorithm to output an $\eps$-optimal arm with probability at least $1 - \delta$ for any MAB instance. Here, an arm $i$ is considered $\eps$-optimal if $\mu_{i^*} - \mu_i < \eps$.
		
		The first tight bound for $(\eps, \delta)$-BAI was achieved by a median elimination algorithm in \cite{EMM06}. 
		
		\begin{algorithm}[h]
			\caption{Median Elimination(\cite{EMM06})}
			\label{algo:me}
			\Input{a set of arms $S$ of size $n$ and $(\eps,\delta)$}\\
			\Output{an arm}
			\begin{algorithmic}[1]
				\Procedure{\me}{{$\eps,\delta,S$}}
				\State Set $S_1=S,\eps_1=\eps/4,\delta_1=\delta/2, \ell=1$
				\While{$\abs{S_{\ell}}>1$}
				\State Sample every arm $a\in S$ for $\frac{4}{\eps_\ell^2}\log \frac{3}{\delta_\ell}$ times, and let $\hat{p}_a^\ell$ denote its empirical mean
				\State Find the median of $\hat{p}_a^\ell$, denoted by $m_\ell$
				\State Update: $S_{\ell+1}=S_\ell\backslash \set{a: \hat{p}_a^\ell <m_\ell}$
				\State Update: $\eps_{\ell+1}=\frac{3}{4}\eps_\ell, \delta_{\ell+1}=\frac{\delta_\ell}{2},\ell=\ell+1$
				\EndWhile
				\EndProcedure
			\end{algorithmic}
		\end{algorithm}

		\begin{proposition}[\cite{EMM06}, Theorem 10]\label{prop:me}
			Taking $\eps$ and $\delta$, and an arm set $S$ as input, \Cref{algo:me} is an $(\eps, \delta)$-BAI algorithm with sample size $T= \+O\tp{\frac{n}{\eps^2}\log\frac{1}{\delta}}$.
		\end{proposition}

		\subsection{Best Arm Retention}\label{sec:prelim-BAR}
		
		Given an MAB instance $\nu$, the \emph{Best Arm Retention} (BAR) problem involves retaining $m$ arms $S_T$ out of $n$ after $T$ samples, ensuring that the best arm is included in the retained set $i^*\in S_T$. Similar to BAI, in our paper the sample size $T$ is a  stopping time with respect to $\tp{\+F_t}_{t\in\bb N}$. Correspondingly, the $(\eps, \delta)$-PAC version of BAR, denoted as $(\eps, \delta)$-BAR, requires that the retained $m$ arms contain at least one $\eps$-optimal arm with probability at least $1 - \delta$.
		
		In practice, to achieve low regret in a multiple-pass streaming algorithm, it suffices for the \emph{expected gap}\footnote{We use \emph{mean gap} to refer to the mean difference between $i^*$ and the other fixed arm, and \emph{expected  gap} to denote the expected difference in means between $i^*$ and the optimal arm of an arm subset, where the randomness of the expectation arises from the arm subset.} between the mean of the best arm and the optimal arm in the retained $m$ arms to be small. We define $r$-BAR as the problem to guarantee that the expected gap is at most $r$, namely $\E{\mu_{i^*} - \max_{i \in S_T} \mu_i }< r.$
		
		The regret for a $r$-BAR algorithm, $R(n)=T\mu_{i^*}-\E{\sum_{t=1}^T r_t}$, is similar to the regret defined in \Cref{ssec:mab} except that $T$ is a stopping time.  When referring to the sample (or regret) complexity of $r$-BAR, we denote the minimum samples (or regret) required by any algorithm capable of solving $r$-BAR for any MAB instance.

		\section{An $(\eps,\delta)$-PAC Algorithm for BAR}
		
		In this section, we will design a simple algorithm with the assistance of the median elimination algorithm to establish an upper bound, followed by a lower bound based on likelihood ratio.

		\subsection{Upper Bound for $(\eps, \delta)$-BAR}
		
		Our algorithm presented in Algorithm \ref{algo:bar} is straightforward. We first uniformly at random choose $n-m+1$ arms from the set of $n$ arms. Next, we execute the \me algorithm (\Cref{algo:me}) to obtain an $\eps$-optimal arm (with respect to the chosen arms) with probability $1 - \frac{n}{n-m+1}\cdot\delta$. Finally, we output this arm along with the remaining $m-1$ arms that were not chosen in the first stage.
		
		\begin{algorithm}[ht]
			\caption{$(\eps, \delta)$-PAC Algorithm for BAR}
			\label{algo:bar} 
			\Input{$\eps, \delta, m$, arm set $S$}\\
			\Output{$m$ arms}
			\begin{algorithmic}[1]
				\State Choose $n-m+1$ arms denoted as $S^\prime$ from $n$ arms uniformly at random.
				\State Run the median elimination algorithm $$i^\prime = \me(\eps, \frac{n}{n-m+1}\delta, S^\prime).$$
				\State \Return $S \setminus S^\prime \cup \{i^\prime\}$.
			\end{algorithmic}
		\end{algorithm}
		
		\begin{theorem}[Part of Theorem \ref{thm:bar}]
			Algorithm \ref{algo:bar} is an $(\eps, \delta)$-BAR algorithm with a sample size of $$\+O\tp{\frac{n-m+1}{\eps^2}\log\frac{n-m+1}{n\delta}}.$$
		\end{theorem}
		
		\begin{proof}
			Our algorithm fails if and only if: 
			\begin{itemize}
				\item[(1)] the optimal arm was chosen in the first stage, and 
				\item[(2)] the procedure \me  failed to return a nearly optimal arm.
			\end{itemize}
			Therefore, the failure probability is given by 
			$$\frac{n-m+1}{n} \cdot \frac{n}{n-m+1}\cdot\delta = \delta.$$
			Furthermore,  by \Cref{prop:me} the sample complexity due to the \me procedure
			$\+O\tp{\frac{n-m+1}{\eps^2}\log\frac{n-m+1}{n\delta}}$.
		\end{proof}
		
		
		\subsection{Lower Bounds}\label{subsec:PAC-lb}
		Recall the mean vector $\mu=(\mu_1, \mu_2,\cdots, \mu_n)$ denotes a MAB instance, where each arm $i$ follows a Bernoulli distribution with mean $\mu_i$.  Consider the following $n$ instances: $\@H_1=(\frac{1}{2}+\eps,\frac{1}{2},\dots,\frac{1}{2})$, and for $j\neq 1$, $\@H_j$ differs from $\@H_1$ only in $\@H_j(j)=\frac{1}{2}+2\eps$. We use $\Pr[i]{\cdot}$ and $\E[i]{\cdot}$ to denote probability and expectation of the algorithm running on instance $\@H_i$.
		
		At a high level, if an algorithm outputs arm $j$ as the best arm with a higher probability in $\@H_j$ than in $\@H_1$, then this algorithm, to some extent, distinguishes the two instances,  indicating that arm $j$ should be pulled enough times.

		The well-known lower bound for $(\eps,\delta)$-BAI, $\Omega\tp{\frac{n}{\eps^2}\log\frac{1}{\delta}}$, is tight but with the restriction of $\delta<0.5$~\cite{MT04}. The lower bound proof techniques from previous literature (e.g. ~\cite{MT04}) are mainly based on the following observation: given two instances $\@H_1$ and $\@H_j$, any proper algorithm should retain arm $j$ with probability at least $1-\delta$ on $\@H_j$, but at most $\delta$ on $\@H_1$.  When $1-\delta>\delta$, enough pulls  for arm $j$ are required to distinguish between these two instances, which necessitates $\delta<\frac{1}{2}$. However, when $\delta \geq \frac{1}{2}$, a more refined argument is required.

		We will begin with a warm-up lower bound for BAR with $m=1$ (or equivalently, BAI) to demonstrate how to eliminate this restriction in \Cref{thm:bai}. For $\delta\geq0.5$, we know that $\Theta\tp{\frac{n}{\eps^2}\log\frac{1}{\delta}}=\Theta\tp{\frac{(1-\delta)n}{\eps^2}}$. Therefore our result shows that the lower bound  $\Omega\tp{\frac{n}{\eps^2}\log\frac{1}{\delta}}$ holds for almost all feasible $\delta$. Subsequently, we will extend this method to BAR with general $m$.
		
		\subsubsection{BAI: Warm-up}
		\begin{theorem}\label{thm:bai}
			For any  $(\eps,\delta)$-BAI algorithm satisfying $1-\delta= \frac{1+\Omega(1)}{n}$ and $\eps\leq \frac{1}{8}$, the  sample size when running on $\@H_1$ is 
			$$\E[1]{T}=\Omega\tp{\frac{(1-\delta)n-1}{\eps^2}}.$$
		\end{theorem}
		We will prove \Cref{thm:bai} in this part.
		Given any algorithm $\+A$, let $$B=\set{i\in\set{2,\dots,n}:\Pr[1]{\+A \text{ outputs arm } i}\leq \frac{\delta}{n-k}},$$ where $k$ is an integer to be determined later such that $\frac{\delta}{n-k}\leq 1-\delta$. It is evident that $\abs{B}\geq k$. Otherwise, there are $n-k$ arms $j\in \set{2,\dots,n-1}$ satisfying $\Pr[1]{\+A\mbox{ outputs arm } j}>\frac{\delta}{n-k}$ and this contradicts that $\+A$ is an $(\eps,\delta)$-PAC algorithm. However, for each $i\in B$, we have $\Pr[i]{\+A \text{ outputs arm } i}\geq 1-\delta$. The following lemma obtained by likelihood ratio shows that arm $i$ must be pulled enough times. The proof is given here for completeness.
		\begin{lemma}[\cite{KCG16}, Lemma 1]\label{lem:likely-ratio}
			For any two MAB instances $\mu, \mu^\prime$ with $n$ arms, and for any algorithm with almost-surely finite stopping time $T$ and event $\+E\in \+F_T$, 
			$\sum_{i=1}^n \tp{\E[\mu]{T_i}\cdot \fd(\mu_i,\mu^\prime_i)}\geq 
			\fd(\Pr[\mu]{\+E}, \Pr[\mu^\prime]{\+E}).$
		\end{lemma}
		\begin{proof}
			Let $T_i(s)$ denote the index of the $s$-th pull of arm $i$ for $s\leq T_i$. Define the log-likelihood ratio $L_T(a_1,r_1,a_2,r_2,\dots,a_T,r_T)=\log\frac{\Pr[\mu]{a_1,r_1,a_2,r_2,\dots,a_T,r_T}}{\Pr[\mu^\prime]{a_1,r_1,a_2,r_2,\dots,a_T,r_T}}$, abbreviated as $L_T$ when the context is clear.   By applying the chain rule to $L_T$, we have 
			\begin{align*}
				L_T&= \log \frac{\prod_{t=1}^T\Pr[\mu]{a_t\vert \+F_{t-1}}\cdot \Pr[\mu]{r_t\vert \+F_{t-1}, a_t}}{\prod_{t=1}^T\Pr[\mu^\prime]{a_t\vert \+F_{t-1}}\cdot \Pr[\mu^\prime]{r_t\vert \+F_{t-1}, a_t}}\\
				&=\sum_{t=1}^T\log\frac{\Pr[\mu]{r_t\vert a_t}}{\Pr[\mu^\prime]{r_t\vert a_t}} =\sum_{i=1}^n\sum_{s=1}^{T_i}\log\frac{\Pr[\mu]{r_{T_i(s)}\vert a_{T_i(s)}}}{\Pr[\mu^\prime]{r_{T_i(s)}\vert a_{T_i(s)}}},
			\end{align*}
			where the second equality follows from $\Pr[\mu]{a_t\vert \+F_{t-1}}=\Pr[\mu^\prime]{a_t\vert \+F_{t-1}}$ and that $r_t$ is independent of $\+F_{t-1}$ conditioned on $a_t$.
			With $\E[\mu]{ \log\frac{\Pr[\mu]{r_{T_i(s)}\vert a_{T_i(s)}}}{\Pr[\mu^\prime]{r_{T_i(s)}\vert a_{T_i(s)}}}}=\fd(\mu_i,\mu_i^\prime)$, we apply Wald's Lemma (see e.g. \cite{siegmund13}) to $\sum_{i=1}^n\sum_{s=1}^{T_i}\log\frac{\Pr[\mu]{r_{T_i(s)}\vert a_{T_i(s)}}}{\Pr[\mu^\prime]{r_{T_i(s)}\vert a_{T_i(s)}}}$ to obtain:
			\begin{equation}\label{eqn:likeli-lhs}
				\E[\mu]{L_T} =\sum_{i=1}^n\E[\mu]{T_i}\fd\tp{\mu_i,\mu_i^\prime}.
			\end{equation}
			
			The remaining task is to  prove $\E[\mu]{L_T}\geq \fd(\Pr[\mu]{\+E},\Pr[\mu^\prime]{\+E})$ for any event $\+E\in\+F_T$.
			%
			Apply \Cref{fact:log-sum} to \[\E[\mu]{L_T\vert \+E}=\sum_{\omega\in \+E}\frac{\Pr[\mu]{\omega}}{\Pr[\mu]{\+E}}\log \frac{\Pr[\mu]{\omega}}{\Pr[\mu']{\omega}}\] to obtain $\E[\mu]{L_T\vert \+E}\geq \log\frac{\Pr[\mu]{\+E}}{\Pr[\mu^\prime]{\+E}}$.
			Similarly, $\E[\mu]{L_T\vert \bar{\+E}}\geq \log\frac{\Pr[\mu]{\bar{\+E}}}{\Pr[\mu^\prime]{\bar{\+E}}}$.
			Hence, we conclude
			\begin{align*}
				\E[\mu]{L_T}&=\Pr[\mu]{\+E} \E[\mu]{L_T\vert \+E}+\Pr[\mu]{\bar{\+E}}\E[\mu]{L_T\vert \bar{\+E}}\\
				&\geq \Pr[\mu]{\+E} \log\frac{\Pr[\mu]{\+E}}{\Pr[\mu^\prime]{\+E}}+  \Pr[\mu]{\bar{\+E}} \log\frac{\Pr[\mu]{\bar{\+E}}}{\Pr[\mu^\prime]{\bar{\+E}}}\\
				&=\fd(\Pr[\mu]{\+E},\Pr[\mu^\prime]{\+E}),
			\end{align*}
			which completes our proof in conjunction with \Cref{eqn:likeli-lhs}.
			
		\end{proof}

		Here we only consider the algorithm with almost-surely finite stopping time. Otherwise  the sample complexity is infinite and the theorem obviously holds.  Therefore, for any $i\in B$, we can apply \Cref{lem:likely-ratio} to instances $\@H_1$ and $\@H_i$ with $\+E_i=\set{\+A \text{ outputs arm } i}$ to obtain
		$$\E[1]{T_i}\cdot \fd\tp{0.5,0.5+2\eps}\geq \fd\tp{\Pr[1]{\+E_i}, \Pr[i]{\+E_i}}.$$
		Since $\fd(0.5,0.5+2\eps)\leq 12\eps^2$ and $\fd\tp{\Pr[1]{\+E_i}, \Pr[i]{\+E_i}}\geq  \fd\tp{\frac{\delta}{n-k},1-\delta}$ by \Cref{fact:kl}, then
		$$12\eps^2\cdot\E[1]{T_i}\geq \fd\tp{\frac{\delta}{n-k},1-\delta}.$$
		Summing up all $i\in B$, we have 
		$$12\eps^2\E[1]{T}\geq k\cdot \fd\tp{\frac{\delta}{n-k},1-\delta}.$$
		Here we choose $k=n-\frac{\delta}{\frac{1-\delta}{2}+\frac{1}{2n}}=\Omega(n)$, and thus $\frac{\delta}{n-k}=\frac{1-\delta}{2}+\frac{1}{2n}$. 
		The following lemma assists us in bounding the KL divergence:
		\begin{lemma}\label{lem:bai-kl}
			$\fd(\frac{1-\delta}{2}+\frac{1}{2n},1-\delta)=\Omega\tp{\frac{1-\delta}{2}-\frac{1}{2n}}$ if $1-\delta=\frac{1+\Omega(1)}{n}$.
		\end{lemma}
		\begin{proof}
			By definition,
			\begin{align*}
				&\phantom{{}={}} \fd(\frac{1-\delta}{2}+\frac{1}{2n},1-\delta)\\
				&= \tp{\frac{1-\delta}{2}+\frac{1}{2n}}\log \frac{{\frac{1-\delta}{2}+\frac{1}{2n}}}{1-\delta}+\tp{1-\frac{1-\delta}{2}-\frac{1}{2n}}\log \frac{{1-\frac{1-\delta}{2}-\frac{1}{2n}}}{\delta}\\
				&= \log\tp{1+ \frac{\frac{1-\delta}{2}-\frac{1}{2n}}{\delta}}+\tp{\frac{1-\delta}{2}+\frac{1}{2n}}\tp{\log \tp{1-\frac{\frac{1-\delta}{2}-\frac{1}{2n}}{1-\delta}}  +\log\tp{1-\frac{\frac{1-\delta}{2}-\frac{1}{2n}}{\frac{1+\delta}{2}-\frac{1}{2n}}  } }\\
				&\geq  \tp{\frac{1-\delta}{2}-\frac{1}{2n}}\tp{\frac{1}{ \frac{1+\delta}{2}-\frac{1}{2n}}- \tp{\frac{1-\delta}{2}+\frac{1}{2n}}\tp{\frac{1}{\frac{1-\delta}{2}+\frac{1}{2n}}\cdot\tp{1-\frac{\frac{1-\delta}{2}-\frac{1}{2n}}{2(1-\delta)}}+\frac{1}{\delta} } }\\
				&= \tp{\frac{1-\delta}{2}-\frac{1}{2n}}\tp{\frac{1}{\frac{1+\delta}{2}-\frac{1}{2n}}-\frac{3}{4}-\frac{1}{4(1-\delta)n}-\frac{\frac{1-\delta}{2}+\frac{1}{2n}}{\delta}}\\
				&=\Omega{\tp{\frac{1-\delta}{2}-\frac{1}{2n}}},
			\end{align*}
			where the inequality follows from  (a) \& (b) of \Cref{fact:log-ineq}.
		\end{proof}

		Therefore sample complexity $\E[1]{T}=\Omega\tp{\frac{(1-\delta)n-1}{\eps^2}}$ if $1-\delta=\frac{1+\Omega(1)}{n}$.
		
		\begin{remark}
			\Cref{lem:likely-ratio} is more powerful  for the BAI problem than the previously used tools in the bandit community. The first classic tool is the Bretagnolle–Huber inequality, which is stated as follows:
			
			For any event $\+E\in \+F$ and any two probability distribution $\-P$ and $\-Q$ on a measurable space $(\Omega,\+F)$,  
			$$\-P(\+E)+\-Q(\bar{\+E})\geq \frac{1}{2}\exp\tp{-\DKL \tp{\-P,\-Q}},$$
			where $\DKL$ is the KL divergence.
			It is straightforward to verify that $\DKL\tp{\@H_0,\@H_i}=\Theta\tp{\eps^2 \cdot\E[0]{T_i}}$. However, finding an event such that $\@H_0(\+E_i)+\@H_i(\bar{\+E}_i)\leq \frac{\delta}{2}$ (or even smaller) is challenging, which limits the effectiveness of this tool to cases where $\delta < 0.5$.
			
			Another frequently used method  is based on   a likelihood ratio argument (see e.g. \cite{MT04}).  For any instance $\@H_i$ where $i\in B$, assume that $\E[0]{T_i}<L$ where $L$ is the lower bound on $\E[0]{T_i}$ that we want to prove by contradiction.  Define the event $A=\set{\+A \text{ outputs arms other than arm } i}$, and we know that $\Pr[0]{A}\geq 1-\frac{\delta}{n-k}$ and $\Pr[i]{A}\leq \delta$. 
			We also need to define an additional  event $C$ with certain concentration properties such that  both $T_i < cL$ ( where $c>1$ is a constant or parameter to be determined) and the average rewards of arm $i$, when pulled for the $t$-th time for any $t \leq T_i$, are close to $1/2$. This allows us to lower bound the likelihood ratio $\frac{\Pr[i]{\omega}}{\Pr[0]{\omega}} \geq r$ for any $\omega \in A \cap C$.
			Hence
			$$\Pr[i]{A}\geq \Pr[i]{A\cap C}\geq r   \Pr[0]{A\cap C}.$$

			When $\Pr[0]{C}$ is small, $r$ can be large, but $\Pr[0]{A \cap C}$ is small. Conversely, when $\Pr[0]{C}$ is large, $r$ is small, but $\Pr[0]{A \cap C}$ can be large. The main challenge is to choose an appropriate $C$ to balance these two terms so that their product is greater than $\delta$, leading to a contradiction. Typically, $C$ is selected based on some threshold, and as a result, $\Pr[0]{C}$ tends to be small. This makes the lower bound derived from this method less tight than the one obtained from \Cref{lem:likely-ratio}, as the method only uses a small subset of the samples and is, therefore, less precise.
			
			In contrast, \Cref{lem:likely-ratio} utilizes all the available samples and only applies Jensen's inequality to the function $f(x) = x\log x$ with the values $\frac{\Pr[0]{\omega}}{\Pr[i]{\omega}}$ for $\omega \in \+E_i$ and $\omega \in \bar{\+E}_i$. This makes it a much more robust and accurate tool in comparison.
			
			Moreover, as demonstrated in the next section,   \Cref{lem:likely-ratio} is more concise, making it easier to pair with \Cref{lem:kl-sum}. In contrast, the  likelihood ratio argument is relatively cumbersome, making it harder to combine with \Cref{lem:kl-sum}.
			
		\end{remark}

		\subsubsection{Lower Bound for $(\eps,\delta)$-BAR}
		
		In this part, we establish a more general lower bound for the $(\eps,\delta)$-PAC algorithms for best arm retention (BAR) by refining arguments in the proof of \Cref{thm:bai}. Similar to the proof for BAI, we only need to consider the algorithm with the almost-surely finite stopping time for the sample complexity.  The following theorem is stronger than the lower bound in \Cref{thm:bar}.

		\begin{theorem}\label{thm:bar-lb}
			For any $(\eps,\delta)$-BAR algorithm  with almost-surely finite stopping time such that $\eps\leq \frac{1}{8}$ and $\delta\leq \frac{n-m}{n}(1-\beta)$, where $\beta\in [0,1)$ is a constant, its sample complexity on the input $\@H_1$ satisfies
			$$\E[1]{T-T_1}\geq \frac{\beta(n-m-\delta)}{2\eps^2}\log\frac{n-m-\delta}{(n-1)\delta}.$$
		\end{theorem}
		We reserve the notations introduced in the previous parts except $$\+E_i=\set{\+A \text{ retains arm }i}.$$ For any algorithm, we have $m=\E[1]{\sum_{i=1}^n\ind{\+E_i}}=\sum_{i=1}^n\Pr[1]{\+E_i}$.
		
		If we directly apply an argument similar to that in the proof of~\Cref{thm:bai} here and   redfine the sets as $$B=\set{i\in\set{2,\dots,n}:\Pr[1]{\+E_i}\leq \frac{m-(1-\delta)}{n-k}}$$ and   $$\bar{B}=\set{i\in\set{2,\dots,n}:\Pr[1]{\+E_i}> \frac{m-(1-\delta)}{n-k}},$$ then we can conclude  that there are at least $k$ arms retained with probability at most $\frac{m-(1-\delta)}{n-k}$, meaning $\abs{B}\geq k$. Otherwise, if $\abs{B}\leq k-1$, then $\abs{\bar B}\geq n-k$ since $\abs{B}+\abs{\bar B}=n-1$. Summing up the probabilities over all $n$ arms gives us $\sum_{i=1}^n\Pr[1]{\+E_i}>(n-k)\cdot\frac{m-(1-\delta)}{n-k}+1-\delta=m$, which leads to a contradiction. 
		
		Next, applying \Cref{lem:likely-ratio} to each pair of  arms in $B$ and arm $1$, and summing  over the at least $k$ arms in $B$, we  obtain the  lower bound 
		$$12\eps^2 \E[1]{T-T_1}\geq k\cdot \fd\tp{\frac{m-(1-\delta)}{n-k},1-\delta}.$$
		We should choose a $k$ satisfying $\frac{m-(1-\delta)}{n-k}<1-\delta$ to maximize $k\cdot \fd\tp{\frac{m-(1-\delta)}{n-k},1-\delta}$. Consider a special case with $m=n-1$ and $\delta=\frac{1}{2n}$, where we can only choose $k=1$. Thus $$k\cdot \fd(\frac{m-(1-\delta)}{n-k},1-\delta)=\fd(1-\frac{1-\frac{1}{2n}}{n-1},1-\frac{1}{2n})=\Theta\tp{\frac{1}{n}},$$ which leads to $\E[1]{T}=\Omega\tp{\frac{1}{\eps^2n}}$. However, the upper bound is $T\leq \+O\tp{ \frac{1}{\eps^2}}$ in this case. 
		
		The above analysis is too pessimistic as we classify suboptimal arms (those in $B$ and those not in $B$) via a single threshold. The following lemma allows us to argue about their sum directly.
		
		\begin{lemma}\label{lem:kl-sum}
			For any $x_1,x_2\dots,x_n \in[0,1]$ with average $a\defeq\frac{\sum_{i}x_i}{n}< b\in[0,1]$, then 
			$\sum_{i:x_i<b}\fd(x_i,b)\geq n\cdot \fd(a,b).$
		\end{lemma}
		
		\begin{proof}
			
			Recall that $\fd(\cdot,y)$ is convex for any fixed $y$ in \Cref{fact:kl}. Let $S=\set{i:x_i<b}$ and $k=\abs{S}$. By the convexity of $\fd(\cdot,b)$, we have 
			$\frac{1}{k}\sum_{i\in S}\fd(x_i,b)\geq \fd\tp{\frac{\sum_{i\in S}x_i}{k},b}.$
			Since $\fd(x,b)>\fd(y,b)$ if $x<y<b$ in \Cref{fact:kl},
			$$\sum_{i\in S}\fd(x_i,b)\geq k\cdot \fd\tp{\frac{\sum_{i\in S}x_i}{k},b}\geq k\cdot \fd\tp{\frac{an-(n-k)b}{k},b}.$$
			Using the convexity of $\fd(\cdot,b)$ again, we get 
			$$\frac{k}{n}\cdot \fd\tp{\frac{an-(n-k)b}{k},b}+\frac{n-k}{n}\cdot \fd(b,b)\geq \fd(a,b),$$
			which implies $k\cdot \fd\tp{\frac{an-(n-k)b}{k},b}\geq n\cdot \fd(a,b)$ since $\fd(b,b)=0$.
			
		\end{proof}
		
		Armed with \Cref{lem:kl-sum}, we can sum up all $i$ such that $\Pr[1]{\+E_i}\leq 1-\delta$ to which \Cref{lem:likely-ratio} can be applied.
		\begin{align*}
			12\eps^2 \E[1]{T-T_1} &\geq \sum_{i:\Pr[1]{\+E_i}<1-\delta}\fd\tp{\Pr[1]{\+E_i},\Pr[i]{\+E_i}} \quad\quad(\text{By \Cref{lem:likely-ratio}})\\
			&\geq \sum_{i:\Pr[1]{\+E_i}<1-\delta}\fd\tp{\Pr[1]{\+E_i},1-\delta} \quad\quad(\text{By \Cref{fact:kl}})\\
			&\geq (n-1)\cdot \fd\tp{\frac{m-(1-\delta)}{n-1},1-\delta}. \quad\quad(\text{By \Cref{lem:kl-sum}})
		\end{align*}
		
		Finally, we use the following lemma to help analyze the KL divergence:
		\begin{lemma}\label{lem:bar-kl}
			For any $0<a<b<1$, 
			$\fd(b,a)\geq \tp{1-\frac{1}{1+r/(2+r)}}b\log\frac{b}{a},$ where $r\defeq \frac{b-a}{a}$.
		\end{lemma}
		
		\begin{proof}
			By definition of the KL divergence,
			$\fd(b,a)=b\log\frac{b}{a}+(1-b)\log\frac{1-b}{1-a}.$
			
			\noindent By \Cref{fact:log-ineq} (b) \& (c),
			$$b\log\frac{b}{a}=b\log\tp{1+\frac{b-a}{a}}\geq (b-a)\tp{1+\frac{(b-a)/a}{2+(b-a)/a}}$$
			and 
			$$(1-b)\log\frac{1-b}{1-a}=(1-b)\log\tp{1+\frac{a-b}{1-a}}\geq -(b-a)\tp{1-\frac{b-a}{2(1-a)}}. $$
			Therefore, 
			\begin{align*}
				&\phantom{{}={}}	\fd(b,a) \\
				&= \tp{1-\frac{1}{1+r/(2+r)}}b\log\frac{b}{a}+\frac{1}{1+r/(2+r)}b\log\frac{b}{a}+(1-b)\log\frac{1-b}{1-a}\\
				&\geq \tp{1-\frac{1}{1+r/(2+r)}}b\log\frac{b}{a} +(b-a)-(b-a)\tp{1-\frac{b-a}{2(1-a)}}\\
				&\geq\tp{1-\frac{1}{1+r/(2+r)}}b\log\frac{b}{a}.
			\end{align*}
		\end{proof}
		%
		
		Now we are  ready to bound $\fd(\frac{m-(1-\delta)}{n-1},1-\delta)$. Let $\delta= \frac{n-m}{n}(1-\beta')$ where $\beta'>\beta$ is a universal constant.
		\begin{align*}
			\fd\tp{\frac{m-(1-\delta)}{n-1},1-\delta} 
			&= \fd\tp{\frac{n-m-\delta}{n-1},\delta} \quad \quad(\fd(x,y)=\fd(1-x,1-y))\\
			&=\fd\tp{\frac{n-m}{n}\tp{1+\frac{\beta'}{n-1}},\frac{n-m}{n}(1-\beta')}\\
			&=: \fd(B,A).
		\end{align*}
		Here $\frac{B-A}{A}=\frac{\beta'/(n-1)+\beta'}{1-\beta'}\geq \frac{\beta'}{1-\beta'} =: r'$. Therefore,
		\begin{align*}
			\fd\tp{\frac{m-(1-\delta)}{n-1},1-\delta}& \geq \tp{1-\frac{1}{1+r'/(2+r')}}{\frac{n-m-\delta}{n-1}\log\frac{n-m-\delta}{(n-1)\delta}}\\
			&\geq {\frac{\beta'(n-m-\delta)}{2(n-1)}\log\frac{n-m-\delta}{(n-1)\delta}},\\
			&\geq {\frac{\beta(n-m-\delta)}{2(n-1)}\log\frac{n-m-\delta}{(n-1)\delta}}.
		\end{align*}
		Thus, $\E[1]{T-T_1}\geq \frac{\beta(n-m-\delta)}{2\eps^2}\log\frac{n-m-\delta}{(n-1)\delta}$.

		\begin{remark}\label{remark}
			This approach encounters limitations when $\delta$ approaches the boundary $\frac{n-m}{n}$, specifically when $\frac{n-m}{n}-\delta=o\tp{\frac{1}{n}}$. For instance, consider the scenario where $m=n-1$ and $\delta=\frac{1}{n}-\frac{1}{n^2}$. In this case, $$\fd(\frac{m-(1-\delta)}{n-1},1-\delta)=\fd(\frac{n-1}{n}-\frac{1}{n^2(n-1)},\frac{n-1}{n}+\frac{1}{n^2})=\Theta\tp{\frac{1}{n^2}}.$$ Consequently, the resulting lower bound is $\Omega\tp{\frac{1}{\eps^2 n}}$.
			
			Suppose there exists an algorithm that achieves this lower bound, making it an $(\eps, \frac{1}{n}-\frac{1}{n^2})$-BAR algorithm with a sample complexity of $c\frac{1}{\eps^2 n}$, where $c$ is a universal constant independent of $n$. However, as $n$ grows sufficiently large, such that $c\frac{1}{\eps^2 n}< 1$, this algorithm is paradoxical. It allows for retaining an $\eps$-optimal arm with a higher probability than $\frac{n-1}{n}$ but without exploration, which is logically impossible.
		\end{remark}
		
		\section{$r$-BAR}
		Recall that $r$-BAR requires  the mean difference  between the best arm from $n$ arms and the optimal arm from the retained pool of size $m<n$ is less than $r$ in expectation. In this section, we study both the sample complexity and the minimum regret of this problem. Our results reveal some connections and distinctions between these two optimization objectives.
		
		\subsection{Sample Complexity for $r$-BAR}
		\subsubsection{Exploration Algorithm for  $r$-BAR}
		
		Directly adapting the $(\eps,\delta)$-PAC algorithm to a $r$-BAR setting would imply an expected gap bounded by $\delta + (1-\delta)\eps \leq r$. This translates to $\delta \leq r$ and $\eps \leq 2r$ for $\delta \leq 0.5$. Consequently, the sample complexity of this algorithm becomes $O\left(\frac{n-m}{r^2}\log\frac{n-m}{nr}\right)$. However, when $r$ is small, this bound is not  tight compared to the optimal bound in \Cref{thm:rbar-sample}. To address this, we leverage the insight that a lower expected gap suggests lower regret. Thus, we employ  the procedure  \textsc{MirrorDescent} and choose arms with probabilities proportional to their pull counts. A similar approach has been explored in \cite{BMS09,CHZ23,HYZ24}.
		Let us restate the algorithm for completeness:
		
		\begin{algorithm}
			\Input{arm set $\+S$ of size $n$ and time horizon $T$}\\
			\Output{a good arm}
			\caption{Find best arm with online stochastic mirror descent}\label{algo:findbest} 
			\begin{algorithmic}[1]
				\Procedure{\findbest}{$S,T$}
				\State Run \textsc{MirrorDescent} on $\+S$ with $T$ rounds\;
				\State Compute $T_i,\forall i\in[n]$: the number of times arm $i$ is pulled during $T$ rounds
				\State  Choose arm $i^\prime$ from $\+S$ with probability $\frac{T_{i^\prime}}{T}$\;
				\State \Return arm $i^\prime$
				\EndProcedure
			\end{algorithmic}
			
		\end{algorithm}

		\begin{lemma}\label{lem:findbest}
			Let $i^*$ be the best arm among $\+S$, then for $n$ arms with any mean $\mu$ as input,  \Cref{algo:findbest}  satisfies
			$\E{\mu_{i^*}-\mu_{i^\prime}}\leq\sqrt{\frac{2n}{T}}.$
		\end{lemma}
		\begin{proof}
			Direct calculation yields
			\begin{align*}
				\E{\mu_{i^*}-\mu_{i^\prime}}& =\E{\E{\mu_{i^*}-\mu_{i^\prime}\vert T_1,T_2,\dots,T_n}}\\
				&=\E{\sum_{\text{arm }j\in\+S}\Delta_j\cdot \Pr{i^\prime=j\vert T_1,T_2,\dots,T_n}}\\
				&=\E{\sum_{\text{arm }j\in\+S}\Delta_j\cdot \frac{T_j}{T} } \leq \sqrt{\frac{2n}{T}},
			\end{align*}
			where the last inequality follows from \Cref{prop:osmd}.
		\end{proof}
		
		We can employ the previously described subroutine to devise our final algorithm. Firstly, we randomly select $n-m+1$ arms from the set of $n$ arms (when $m=1$, we select all $n$ arms), denoted as $S^\prime$. We then run \Cref{algo:findbest} with a sufficient number of rounds. Finally, we add the output arm to the remaining unselected arms to form the output set.
		
		\begin{algorithm}
			\caption{Optimal sampling for $r$-BAR}\label{algo:bar-sample} 
			\Input{arm set $\+S$ of size $n\geq m$ and expectation gap $r$}\\
			\Output{$m$ arms}
			\begin{algorithmic}[1]
				\State Sample $n-m+1$ arms, denoted as $S^\prime$, uniformly at random from $\+S$
				\State $i^\prime=\findbest(S^\prime,T^*)$ where $T^*=\frac{2(n-m+2)^3}{(nr)^2}$\;
				\State \Return $\set{i^\prime}\cup S\backslash S^\prime$
			\end{algorithmic}
		\end{algorithm}
		\begin{theorem}[Part of \Cref{thm:rbar-sample}]\label{thm:rbar-sample-ub}
			\Cref{algo:bar-sample} is an algorithm for $r$-BAR with sample complexity $\+O\tp{\frac{(n-m)^3}{(nr)^2}}$.
		\end{theorem}
		\begin{proof}
			The sample complexity is straightforward.  To demonstrate that the expected gap between the best arm in $\set{i^\prime}\cup S\backslash S^\prime$, denoted by $\hat{i}$, and the best arm $i^*$ among all $n$ arms is less than $r$, note that $i^*$ is excluded only if $i^*\in S^\prime$. Thus,
			\begin{align*}
				\E{\mu_{i^*}-\mu_{\hat{i}}} &= \Pr{i^*\notin S^\prime}\E{\mu_{i^*}-\mu_{\hat{i}}\vert i^*\notin S^\prime} + \Pr{i^*\in S^\prime}\E{\mu_{i^*}-\mu_{\hat{i}}\vert i^*\in S^\prime} \\
				&\leq 0+  \frac{n-m+1}{n}\sqrt{\frac{2(n-m+1)}{T^*}}< r,
			\end{align*}
			where the inequality follows from \Cref{lem:findbest}.
		\end{proof}

		\subsubsection{Lower Bound of \Cref{thm:rbar-sample}}\label{sssec:rbar-sample-lb}
		Let $\hat{i}$ be the best arm among the retained $m$ arms. 
		For any $r$-BAR algorithm, we have $\E{\mu_{i^*}-\mu_{\hat{i}}}\leq r$. By Markov inequality, we have $\Pr{\mu_{i^*}-\mu_{\hat{i}}\geq cr}\leq \frac{1}{c}$ for any $c>0$. This implies that an $r$-BAR algorithm is also  a $(cr,\frac{1}{c})$-PAC algorithm for BAR. Thus the sample complexity is bounded below by $\Omega\tp{\frac{n-m}{(cr)^2}\log\frac{c(n-m)}{n}}$, as per \Cref{thm:bar-lb}. Choosing $c=\frac{2n}{n-m}$ completes the proof for the lower bound of  \Cref{thm:rbar-sample}.

		\subsection{Regret Minimization for $r$-BAR}
		
		\subsubsection{Upper Bound}
		
		While a pure exploration algorithm suffices for $r$-BAR, it may yield large regret. For instance, if $i^*$ is not chosen by \Cref{algo:bar-sample} initially, a low-regret algorithm like  \textsc{MirrorDescent} running on the suboptimal arm can still result in significant regret. To address this issue, we first run \findbest on all $n$ arms for a few rounds. We then add the output arm to the randomly chosen $n-m+1$ arms. Subsequently, we run \findbest on these $n-m+2$ arms again, with the subsequent process identical to that of \Cref{algo:bar-sample}, except for retaining the optimal arm from the initial process. Define $L_2=\frac{2(n-m+2)^3}{(n-1)^2r^2}$ and $L_1=\frac{m-2}{n-1}L_2$. Our algorithm outlined in \Cref{algo:bar-regret} follows a similar approach as proposed in \cite{HYZ24}, albeit with distinct objectives.
		\begin{algorithm}
			\caption{Low-regret sampling for BAR}\label{algo:bar-regret} 
			\Input{arm set $\+S$ of size $n\geq m$ and expectation gap $r$}\\
			\Output{$m$ arms}
			\begin{algorithmic}[1]
				\State	$i_1=\findbest(S, L_1)$
				\State Sample $n-m+1$ arms, denoted as $S^\prime$, uniformly at random from $S\setminus{i_1}$
				\State $i_2=\findbest(S^\prime\cup\set{i_1}, L_2)$
				\If{$i_2=i_1$}
				\State Uniformly at random choose $n-m$ arms from $S^\prime\backslash\set{i_2}$ to drop
				\Else
				\State Drop all arms in $S^\prime\backslash\set{i_2}$
				\EndIf
				\State \Return the remaining arms
			\end{algorithmic}
		\end{algorithm}
		\begin{theorem}
			\Cref{algo:bar-regret} is an algorithm for $r$-BAR with regret $\+O\tp{\sqrt{\frac{(n-m)^3}{nr^2}}}$.
		\end{theorem}
		\begin{proof}
			For the case $m=1$, we can easily run \findbest on $n$ arms with $\frac{2n}{r^2}$ rounds.  It is clear that this is a $r$-BAR algorithm by \Cref{lem:findbest} and the regret is no greater than $\frac{2n}{r}$ by \Cref{prop:osmd}. For $m\ge 2$, we run \Cref{algo:bar-regret} and the proof follows a similar structure to that of \Cref{thm:rbar-sample-ub}. Let $\hat{i}$ denote the best arm among the retained arms. Since $i^*$ will be dropped only if $i^*\in S^\prime$, thus 
			\begin{align*}
				\E{\mu_{i^*}-\mu_{\hat{i}}} &= \Pr{i^*\in S^\prime} \E{\mu_{i^*}-\mu_{\hat{i}}\vert i^*\in S^\prime} \\
				&=\Pr{i^*\neq i_1}\Pr{i^*\in S^\prime\vert i^*\neq i_1} \E{\mu_{i^*}-\mu_{\hat{i}}\vert i^*\in S^\prime} \\
				&\leq \frac{n-m+1}{n-1} \E{\mu_{i^*}-\mu_{\hat{i}}\vert i^*\in S^\prime} \\
				&\leq \frac{n-m+1}{n-1}\sqrt{\frac{2(n-m+2)}{L_2}}< r,
			\end{align*}
			where the last inequality follows from \Cref{lem:findbest}.

			Regarding regret, the initial \findbest procedure incurs a regret of $\sqrt{2nL_1}$. In the subsequent step, let $i^\prime$ denote the best arm in $S^\prime\cup{i_1}$. The regret between playing $i^\prime$ and the algorithm over $L_2$ rounds amounts to $\sqrt{2(n-m+2)L_2}$. If $i^\prime$ is not $i^*$, then
			\begin{align*}
				\E{\mu_{i^*}-\mu_{i^\prime}} &= \Pr{i^*\notin S^\prime\cup\set{i_1}}  \E{\mu_{i^*}-\mu_{i^\prime}\vert i^*\notin S^\prime\cup\set{i_1}}\\
				&= \Pr{i_1\neq i^*}\Pr{i^*\notin S^\prime\vert i^*\neq i_1}  \E{\mu_{i^*}-\mu_{i^\prime}\vert i^*\notin S^\prime\cup\set{i_1}}\\
				&\leq \frac{m-2}{n-1}\Pr{i_1\neq i^*} \E{\mu_{i^*}-\mu_{i_1}\vert i^*\notin S^\prime\cup\set{i_1}}\\
				&\overset{\heartsuit}{=} \frac{m-2}{n-1}\Pr{i_1\neq i^*} \E{\mu_{i^*}-\mu_{i_1}\vert i_1\neq i^*}\\
				&= \frac{m-2}{n-1}\E{\mu_{i^*}-\mu_{i_1}} \overset{\clubsuit}{\leq}  \frac{m-2}{n-1}\sqrt{\frac{2n}{L_1}},
			\end{align*}
			
			where $\heartsuit$ follows that $\mu_{i^*}-\mu_{i_1}$ is independent of  $\ind{i^*\notin S^\prime}$ conditioned on $i^*\neq i_1$, and $\clubsuit$ is because of \Cref{lem:findbest}.
			Therefore the regret is 
			\begin{align*}
				\sqrt{2nL_1}+\sqrt{2(n-m+2)L_2)}+ \frac{m-2}{n-1}\sqrt{\frac{2n}{L_1}}L_2
				&\leq \+O\tp{\frac{\sqrt{m(n-m)^3}}{nr}+\frac{(n-m)^2}{nr}}\\
				&\leq \+O\tp{\frac{(n-m)^2}{nr}\tp{1+\sqrt{\frac{m}{n-m}}}}.
			\end{align*}
			When $n-m=\Omega(n)$,  our regret bound is $\+O\tp{\frac{(n-m)^2}{nr}}$.
		\end{proof}
		
		\subsubsection{Proof of the Lower Bound of  \Cref{thm:rbar-regret}}\label{sssec:rbar-regret-lb}
		For the scenario where the algorithm does not almost surely stop within finite time, achieving a large regret lower bound requires more effort. In such cases, we cannot deduce a large regret by infinite sample complexity because the algorithm may continually pull the best arm. To tackle this, we first establish a lower bound for algorithms with an almost-surely finite stopping time, and then reduce any algorithm to this case.

		For the algorithm with almost-surely finite stopping time, similar to the proof of the lower bound in \Cref{sssec:rbar-sample-lb}, an $r$-BAR algorithm also acts as a $(\frac{2nr}{n-m},\frac{n-m}{2n})$-PAC algorithm for BAR. Consequently, it must play the suboptimal arms  $\E[1]{T-T_1}=\Omega\tp{\frac{(n-m)^3}{(nr)^2}}$ times on $\@H_1$ with $\eps=\frac{2nr}{n-m}$. Therefore, the regret by Wald's equation~(see e.g. \cite{siegmund13}) is $\Omega\tp{\frac{2nr}{n-m}\E[1]{T-T_1}}=\Omega\tp{\frac{(n-m)^2}{nr}}$. 
		
		Now assume there exists a $\frac{r}{2}$-BAR algorithm $\+A$ with regret  $o\tp{\frac{(n-m)^2}{nr}}$, and let $T' =\omega\tp{\frac{(n-m)^2}{nr^2}}$ be a fixed number.  We can construct an  algorithm $\+A^\prime$ with finite stopping time as follows: If $\+A$ stops with $T< T^\prime$ and outputs $S_T$, then $\+A^\prime$ simulates it. Otherwise, $\+A^\prime$ stops in the $T^\prime$-th round, chooses an arm $i^\prime$ proportional to the pull times of each arm in $T^\prime$ rounds, similar to the procedure \findbest, and outputs it with $m-1$ randomly chosen arms as $S_{T^\prime}$.
		
		We use $\E[\+A]{\cdot}$ and $\E[\+A^\prime]{\cdot}$ to denote the expectation of the corresponding algorithms running on some MAB instance and let $\hat{i}$ denote the optimal arm among the retained subset arms, similarly for $\Pr[\+A]{\cdot}$ and $\Pr[\+A^\prime]{\cdot}$. We use $T_i$ to denote the number of times that $i$ is pulled and $\+R=\sum_{i=1}^n\Delta_i T_i$.

		It is evident that the regret of $\+A^\prime$ is less than  that of $\+A$ because it may stop earlier.  Now we claim that $\+A^\prime$ is an $r$-BAR algorithm with regret $o\tp{\frac{(n-m)^2}{nr}}$:
		\begin{align*}
			\E[\+A^\prime]{\mu_{i^*}-\mu_{\hat{i}}} &=\Pr[\+A^\prime]{T \geq T^\prime} \E[\+A^\prime]{\mu_{i^*}-\mu_{\hat{i}} \vert T \geq T^\prime} +\Pr[\+A^\prime]{T < T^\prime} \E[\+A^\prime]{\mu_{i^*}-\mu_{\hat{i}} \vert T < T^\prime}\\
			& \leq \Pr[\+A]{T \geq T^\prime} \E[\+A^\prime]{\mu_{i^*}-\mu_{i^\prime} \vert T \geq T^\prime} +\Pr[\+A]{T < T^\prime} \E[\+A]{\mu_{i^*}-\mu_{\hat{i}} \vert T < T^\prime}\\
			&\leq  \Pr[\+A]{T \geq T^\prime} \E[\+A^\prime]{\frac{\+R}{T^\prime} \vert T \geq T^\prime} + \E[\+A]{\mu_{i^*}-\mu_{\hat{i}} }\\
			&\leq  \frac{1}{T^\prime}\Pr[\+A]{T \geq T^\prime} \E[\+A]{\+R \vert T \geq T^\prime} + \frac{r}{2}
			\leq  \frac{1}{T^\prime}\E[\+A]{\+R} + \frac{r}{2}\leq r,
		\end{align*}
		which leads to a contradiction. Hence,  the regret of any $r$-BAR algorithm  is $\Omega\tp{\frac{(n-m)^2}{nr}}$.
		
		\subsection{Difference between Sample Complexity and Regret Minimization}\label{ssec:diff}
		The proof of the lower bound in \Cref{sssec:rbar-sample-lb} reveals that the challenging scenario for $r$-BAR with optimal regret occurs in $\@H_1$ with $\eps=\Theta\tp{\frac{nr}{n-m}}$.  Our analysis in \Cref{subsec:PAC-lb} shows that, on this instance, the requisite number of rounds $T$ is $\Theta\tp{\frac{(n-m)^3}{(nr)^2}}$ in expectation.
		
		If we consider an MAB game with fixed rounds $T=\Theta\tp{\frac{(n-m)^3}{(nr)^2}}$, it is well known that the optimal regret is $\Theta\tp{\sqrt{nT}}=\+O\tp{\frac{(n-m)^2}{nr}\tp{1+\sqrt{\frac{m}{n-m}}}}$, which matches our upper bound in \Cref{thm:rbar-regret}. Previous works have shown that this regret lower bound for MAB problem is achieved on the hard instance $\@H_1$ but with mean gap parameter $\eps^\prime=\Theta\tp{\sqrt{\frac{n}{T}}}=\Theta\tp{\sqrt{\frac{n}{n-m}}\cdot \eps}$ (see \cite{LS20}). This indicates a regret lower bound for \Cref{algo:bar-regret}: if an $r$-BAR algorithm runs for ${T}= \Theta\tp{\frac{(n-m)^3}{(nr)^2}}$ rounds on any instances (as our \Cref{algo:bar-regret} does), then it has to suffer a regret of $\Omega\tp{\frac{(n-m)^2}{nr}\tp{1+\sqrt{\frac{m}{n-m}}}}$ on some instances.
		
		This discrepancy between our regret upper and lower bounds indicates a natural idea to improve the algorithm. Note that $\@H_1$ with $\eps^\prime$ is not the hardest instance for $r$-BAR. An optimal algorithm need not play $\Theta\tp{\frac{(n-m)^3}{(nr)^2}}$ rounds on this instance. It should be more adaptive, sampling for different number of times on different instances, rather than treating all instances equally by handling them as the hardest one.
		
		
		
		In a nutshell, we conjecture that our lower bound is tight, and a more sophisticated algorithm is required to obtain optimal regret upper bounds.
		\begin{conjecture}
			The regret complexity of the $r$-BAR is $\Theta\tp{\frac{(n-m)^2}{nr}}.$
		\end{conjecture}
		\section{Conclusion and Future Work}
		This work initiates from addressing the lower bound for the $(\eps,\delta)$-PAC algorithm for BAI, aiming to eliminate the previously imposed restriction $\delta<0.5$ in the literature. Subsequently, we extend our proof technique to establish lower bounds of the PAC algorithm for BAR, and determine the sample complexity of $(\eps,\delta)$-BAR for nearly all admissible parameters.
		
		In addition to exploring pure exploration methods, we extend our investigation to regret minimization for $r$-BAR. Notably, we conjecture that our derived lower bound is optimal, warranting further investigation into the design of optimal algorithms. 
		
		Another interesting direction is to establish instance-dependent bounds for BAR, rather than focusing on  the PAC algorithm of BAR as we do in this work.
		
		\bibliographystyle{alpha}
		\bibliography{bar}


		
		
		
		%
		%
		%
	\end{document}